%% file: main.tex
\documentclass[10pt,twocolumn,letterpaper]{article}
\usepackage{subcaption}

\usepackage[pagenumbers]{cvpr}

\definecolor{cvprblue}{rgb}{0.21,0.49,0.74}
\usepackage[pagebackref,breaklinks,colorlinks,allcolors=cvprblue]{hyperref}

\usepackage[utf8]{inputenc} % allow utf-8 input
\usepackage[T1]{fontenc}    % use 8-bit T1 fonts
\usepackage{booktabs}       % professional-quality tables
\usepackage{amsfonts}       % blackboard math symbols
\usepackage{nicefrac}       % compact symbols for 1/2, etc.
\usepackage{microtype}      % microtypography
\usepackage{wrapfig}
\usepackage{amsmath,amsfonts}
\usepackage{algorithmic}
\usepackage{algorithm}
\usepackage{array}
\usepackage{textcomp}
\usepackage{stfloats}
\usepackage{url}
\usepackage{graphicx}

\usepackage{amsthm,amsmath}

\newtheorem{theorem}{Theorem}
\newtheorem{lemma}[theorem]{Lemma}

\usepackage{caption}
\usepackage{xcolor, colortbl}
\usepackage{enumitem}
\usepackage{amssymb, latexsym}
\usepackage{color}
\usepackage{float}
\usepackage{mathtools}
\usepackage[algo2e,ruled,noend,linesnumbered]{algorithm2e}
\usepackage[font={small,it}]{caption}
\usepackage{makecell}
\usepackage{adjustbox}
\usepackage{adjustbox}
\usepackage{placeins}
\usepackage{adjustbox} % in preamble
\usepackage{stfloats} % helps place figure* on the first page

\usepackage{multirow}
\usepackage{booktabs}
\usepackage{enumitem}
\usepackage{comment}

\input{include/macros}

\title{See Less, Drive Better: Generalizable End-to-End Autonomous Driving via Foundation Models Stochastic Patch Selection}

\definecolor{authors}{RGB}{50, 50, 180}
\author{
    \textcolor{authors}{Amir Mallak}$^{1}$ \quad
    \textcolor{authors}{Erfan Aasi}$^{2}$ \quad
    \textcolor{authors}{Shiva Sreeram}$^{2}$ \quad
    \textcolor{authors}{Tsun-Hsuan Wang}$^{2}$ \quad
    \textcolor{authors}{Daniela Rus}$^{2}$ \quad
    \textcolor{authors}{Alaa Maalouf}$^{1,2}$  \\
    {\normalsize $^1$\textcolor{magenta}{University of Haifa}} \quad
    {\normalsize $^2$\textcolor{magenta}{CSAIL, MIT}} \\
    {\tt\small Correspondance: mallak002@gmail.com}
}
\vspace{-2cm}
\begin{document}
\newcommand{\rate}{\textsc{rate}}
\newcommand{\methodName}{Stochastic-Patch-Selection}
\newcommand{\method}{SPS}
\twocolumn[{%
\maketitle

\vspace{-2em}
\begin{center}
\includegraphics[width=1\textwidth]{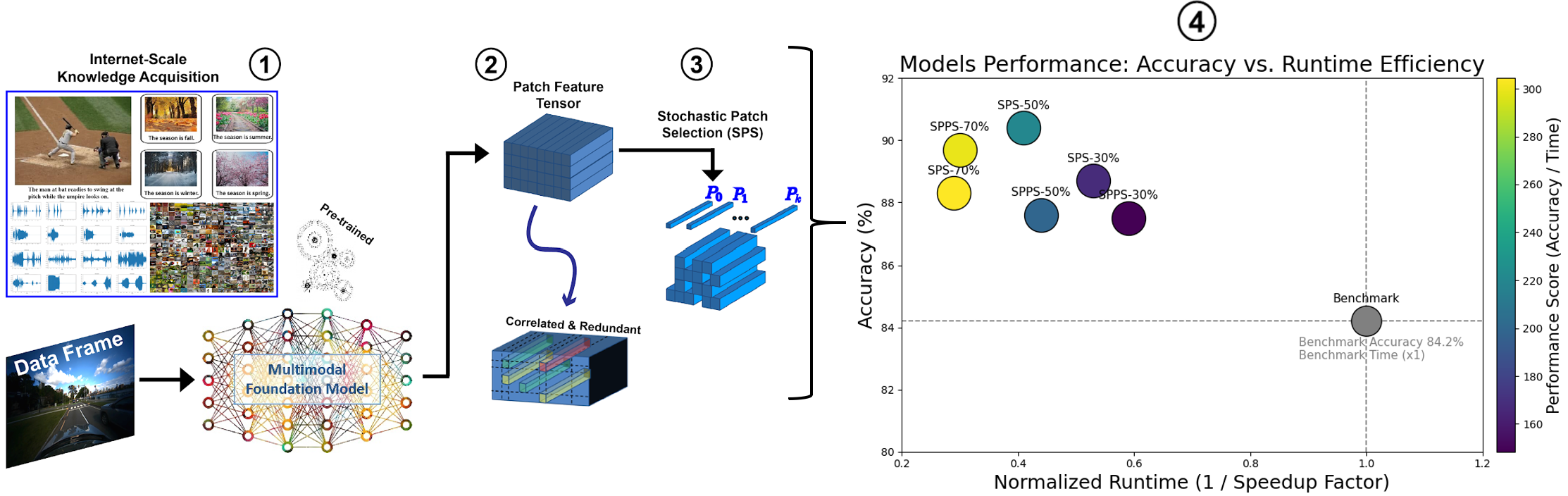}
\captionof{figure}{\textbf{Stochastic Patch Selection (SPS) in a nutshell.} (1) We use large vision-language models to extract patch-level features for input images. (2) These descriptors are often highly redundant and correlated. (3) SPS randomly masks a subset of descriptors, forcing the policy to learn based on different subsets and less correlation, improving efficiency ($\mathbf{2.4\times}$ speedup), generalization ($\mathbf{+6.2\%}$ performance), and enabling plug-and-play integration with downstream policies. In (4), we compare speed vs performance on different variants of \method{} against SOTA.}
\label{fig:teaser}
\end{center}

}]
\vspace{5cm}

\input{sec/0_abstract}
\input{sec/1_intro}
\input{sec/2_related_work}
\input{sec/3_method}
\input{sec/4_experimental_results}
\input{sec/5_conclusions_and_future_work}
\input{sec/6_acknowledgments}

{
    \small
    \bibliographystyle{IEEEtranN}
    \bibliography{main}
}

\input{sec/7_appendix}

\end{document}

%% file: include/macros.tex
\newcommand{\desc}{\texttt{Desc}}

\newcommand{\REAL}{\mathbb{R}}

\newcommand{\softmax}{\texttt{{SoftMax}}}

\newcommand{\Q}{Q^\ell_{\desc(F)}}
\newcommand{\K}{K^\ell_{\desc(F)}}
\newcommand{\V}{V^\ell_{\desc(F)}}
\newcommand{\G}{G^\ell_{\desc(F)}}

%% file: sec/0_abstract.tex
\begin{abstract}
Recent advances in end-to-end autonomous driving show that policies trained on patch‑aligned features extracted from foundation models generalize better to Out‑Of‑Distribution (OOD). We hypothesize that due to the self-attention mechanism, each patch feature implicitly embeds/contains information from all other patches, represented in a different way and intensity, making these descriptors highly redundant. 
We quantify redundancy in such (BLIP2) features via PCA and cross-patch similarity: $90\%$ of variance is captured by $17/64$ principal components, and strong inter-token correlations are pervasive. Training on such overlapping information leads the policy to overfit spurious correlations, hurting OOD robustness. 
We present \methodName{} (\method{}), a simple yet effective approach for learning policies that are more robust, generalizable, and efficient. 
For every frame, \method{} randomly masks a fraction of patch descriptors, not feeding them to the policy model, while preserving the spatial layout of the remaining patches.
Thus, the policy is provided with different stochastic but complete views of the (same) scene: every random subset of patches acts like a different, yet still sensible, coherent projection of the world. The policy thus bases its decisions on features that are invariant to which specific tokens survive. 
Extensive experiments confirm that across all OOD scenarios, our method outperforms the state of the art (SOTA), achieving a $6.2\%$ average improvement and up to $20.4\%$ in closed‑loop simulations, while being $2.4 \times$ faster. 
We conduct ablations over masking rates and patch–feature reorganization, training and evaluating 9 systems, with 8 of them surpassing prior SOTA. 
Finally, we show that the same learned policy transfers to a physical, real-world car without any tuning. 
\end{abstract}

%% file: sec/1_intro.tex
\vspace{-0.4cm}

\section{Introduction and background}
\label{sec:intro&background}

As autonomous‑driving technology matures, end‑to‑end approaches emerged as a leading strategy~\cite{pomerleau1988alvinn,bojarski2016end} embodying a transformative design philosophy: a single model handles everything, from perception to control. This approach (i) dispenses with many of the rigid assumptions imposed on individual submodules, and (ii) evaluates and optimizes performance against a single overarching objective, yielding a more coherent and effective system.

\paragraph{From foundation models to driving polices.} Recently, \citet{Wang2024Drive} showed that extracting patch‑aligned (multimodal) features from a foundation model (FM)~\citep{li2023blip,li2022blip,zhang2022dino}, for each input frame, and training a lightweight policy on those features, markedly improves out‑of‑distribution (OOD) generalization. 
To obtain patch features, they applied \emph{masked} attention at a chosen layer (specifically, at the input to BLIP2’s Q-Former), so the mask guides the attention to focus on a single image patch. Repeating this operation for every patch yields a descriptor per patch. 
Since the descriptors come from the output layer that interfaces with the text encoder, they live in a shared vision–language space, enabling text-driven augmentation and robust policy debugging.

\paragraph{Patch-wise features are redundant.}
We posit that patch descriptors extracted from ViT backbones are intrinsically redundant and correlated.
Before introducing the masked‑attention layer, every patch embedding produced by the sub-ViT backbone has already attended to all other patches, as self‑attention mixes information across tokens; each “patch feature” is no longer a purely local descriptor, as it carries a weighted summary of the entire scene, although the information is focused on this patch.
This redundancy is detrimental when we hand the descriptors to the policy head: (1) Correlated inputs inflate the effective dimensionality of the feature space, making it harder for the policy network to disentangle the truly discriminative signals. (2) Overlapping information encourages over‑fitting; the policy may memorize spurious correlations present in the duplicated context instead of learning robust, spatially grounded cues. Finally, (3) unnecessary redundancy wastes computation and memory, both during training and at inference time.

%-----------------------------------------------------------
% Our Contributions
%-----------------------------------------------------------
\subsection{Our contributions.} 
We first empirically show that patch-features are redundant in Section~\ref {subsec:redundancy_empirics} via similarity overlays, patch-wise correlations, and PCA explained-variance curves. Building on these insights, we propose \textbf{\methodName{} (\method{})}, a lightweight, plug-and-play patch-masking strategy that improves both the efficiency and OOD generalization of policy learning; See Figure~\ref{fig:teaser}. Our main contributions are as follows.

\begin{enumerate}[leftmargin=*,label=(\roman*)]

\item \textbf{\method{}:} At each frame, \method{} randomly drops a fixed proportion of patch descriptors, forwards the untouched descriptors to the policy network, and crucially leaves their grid positions unchanged. Each unmasked subset forms a coherent yet distinct stochastic view of the scene, encouraging the policy to focus on stable/invariant, causally relevant cues instead of spurious correlations.

\item \textbf{Robustness and efficiency gains over SOTA.}
      Across all OOD scenarios, \method{} improves closed‑loop driving success by an average of $6.2\,\%$, and up to $20.4\,\%$, while being $2.4\times$.

\item \textbf{Comprehensive ablations at scale.}
      We train and evaluate 9 autonomous‑driving systems under varied masking ratios and patch‑reorganization schemes, identifying the design choices that maximize robustness; all variants except one surpassed the previous SOTA.

\item \textbf{Sim‑to‑Real generalization.}
      We deployed \method{} on a full-scale autonomous vehicle, showing a seamless transfer from simulation to a real car without tuning.
\end{enumerate}

Notably, as \method{} operates in a shared vision‑language space, it inherits from~\citep{Wang2024Drive} the support for text‑conditioned perturbations/augmentation in the patches embedding space, yielding an additional $1.7\,\%$ performance boost.

%% file: sec/2_related_work.tex
\section{Related work}
\label{sec:related_work}

\paragraph{End-to-end autonomous driving.}
Early work established that neural networks can map raw sensory inputs directly to low-level control~\citep{pomerleau1988alvinn,bojarski2016end}, with subsequent efforts exploring probabilistic objectives and uncertainty-aware control~\citep{amini2019variational}, stability/attention regularization~\citep{wang2023learning}, and safety-aware formulations using control barrier functions~\citep{Xiao2019,xiao2023barriernet}. 
While promising, end-to-end approaches typically demand large, diverse real-world datasets that are costly and risky to collect at scale~\citep{kendall2019learning,amini2022vista}. 
To mitigate this, photorealistic simulators like CARLA, AirSim, and Drake~\citep{dosovitskiy2017carla,shah2018airsim,tedrake2019drake}, and trace-driven interactive simulation such as 
VISTA~\citep{amini2022vista}, have become standard for training and evaluation. %Nevertheless, policies trained in narrow regimes can be brittle under distribution shift. 
A complementary line replaces raw pixels with intermediate visual abstractions (affordances, semantics, lane topology) to simplify control~\citep{muller2018driving,toromanoff2020end,behl2020label}. In this space, ``Drive Anywhere''~\citep{Wang2024Drive} introduced a \emph{foundation-model}-based pipeline that extracts rich, multimodal patch-aligned descriptors and feeds them to a lightweight policy head, showing strong OOD generalization. %under closed-loop testing. 
Our work builds on this FM-based formulation and aims to patch features \emph{reduce redundancy} at train/inference time to improve OOD robustness without modifying or fine-tuning the FM.

\paragraph{Foundation models in robotics and vision.}
Large pretrained vision language models are increasingly used as general-purpose perception and reasoning backbones, from language-conditioned manipulation and planning~\citep{tellex2020robots,bisk2020experience,ahn2022can,brohan2022rt,li2022pre}, video summarizations~\cite{barbara2025prompts}, and 3D/open-world scene understanding~\citep{huang2023audio,ding2023pla,peng2023openscene} to navigation and instruction following~\citep{chahine2024flexendtoendtextinstructedvisual} and follow-and-detect pipelines~\citep{maalouf2023follow,liu2023grounding,ghiasi2022scaling,li2022language,chahine2025decentralized}. 
Their cross-modal capacity also enables generative interfaces that tie vision and language~\citep{ramesh2021zero,crowson2022vqgan,patashnik2021styleclip,ramesh2022hierarchical}. 
Within autonomous driving, some prior work cautions that the temporal reasoning robustness of off-the-shelf VLMs is limited in settings that require strong temporal understanding~\cite{probing}; therefore, deploying them for driving tasks out of the box is not straightforward. This motivates learning downstream driving policies on top of features extracted from pretrained VLMs, rather than relying on end-to-end VLM reasoning. works increasingly leverage language-grounded or explainable representations/features/descriptors for introspection, learning, and counterfactual analysis~\citep{Kim2019-fw,omeiza2021explanations,kuo2022trajectory,tan2023language,zhong2023language,chahine2024flexendtoendtextinstructedvisual,Wang2024Drive}. We follow this trend in spirit, but use only frozen FMs (BLIP/BLIP-2~\citep{li2022blip,li2023blip}, DINO~\citep{caron2021emerging,oquab2023dinov2}) and intervene solely at the patch-descriptor interface to a small policy head—isolating our stochastic selection from FM training dynamics.

\paragraph{Positioning and novelty. } Prior end-to-end policies~\citep{pomerleau1988alvinn,bojarski2016end,amini2019variational,wang2023learning,Xiao2019,xiao2023barriernet} and FM-based driving frameworks~\citep{Wang2024Drive,chahine2024flexendtoendtextinstructedvisual} have not treated \emph{feature redundancy} in FM patch descriptors as a first-class lever for OOD robustness. Such patch embeddings from pretrained models can be obtained by: (i) \emph{region-first} pipelines that segment the image (e.g., SAM/Mask2Former) and then pool features within each mask by an FM encoder (e.g., CLIP)~\citep{kirillov2023segment,zhao2023fast,cheng2022masked,radford2021learning,zhong2022regionclip,jatavallabhula2023conceptfusion}. These methods inherit sensitivity to segmentation quality and add extra compute stages. (ii) \emph{patch-aligned} extraction directly from the backbone without external masks or fine-tuning~\citep{amir2021deep,Wang2024Drive,chahine2024flexendtoendtextinstructedvisual}, preserving spatial layout at token resolution and avoiding segmentation dependencies. 
Our contribution is orthogonal and complementary: we show that even strong patch-aligned descriptors, \emph{are highly redundant}. We therefore introduce \emph{Stochastic Patch Selection (SPS)}, a simple mechanism to handle such correlations, via a \emph{post-FM} token-selection step that preserves spatial layout, reduces token count, and improves closed-loop OOD performance while reducing latency.
SPS is architecture-agnostic, requires no FM fine-tuning. Evidence from PCA and inter-token-correlation analyses, ablations over masking/reorganization variants, and real-world transfer leads to a simple takeaway: for FM-derived patch features in driving, \emph{less can be more}: carefully injected stochastic sparsity improves generalization without any overhead.

%% file: sec/3_method.tex
\vspace{-0.1cm}
\section{Method}
\label{sec:method}
We first provide the preliminaries.
\vspace{-0.1cm}

% =======================
\subsection{Preliminaries}
\label{subsec:prelim}

\paragraph{End-to-end driving as feature-based control.}
We model an autonomous-driving policy as a control function $\phi$ that maps a perception stream $F\!\in\!\REAL^{H\times W\times3}$ of RGB frames to steering and throttle commands $u=\phi(F)$. Rather than operating on raw pixels, we feed $\phi$ a tensor of patch-aligned representations $F'=\desc(F)\!\in\!\REAL^{H'\!\times\!W'\!\times\!D}$ produced by a multimodal foundation model~$\desc$, where $(H',W')$ is the patch grid and $D$ is the channel dimension, i.e.,  $u=\phi\bigl(F'\bigr)$.

We now recall the masked-attention mechanism of \cite{Wang2024Drive}, which enforces spatially selective mixing.

\paragraph{Multimodal patch-wise feature extraction~\citep{Wang2024Drive}.}
Let $\desc$ be a vision transformer of $L$ layers and $N=H'W'$ be the number of non-overlapping patches.
During a forward pass of $\desc$ on $F$, the $\ell$-th self-attention layer yields the  query, key, and value matrices denoted by 
$\Q,\K\in\REAL^{N\times D_k}$ and $\V\in\REAL^{N\times D}$, respectively. To compute the feature of patch $j\!\in\![N]$, \citet{Wang2024Drive} introduced (1) an attention mask \(m^{(j)} \in [0,1]^N\), where \(m^{(j)}_i = 1\) preserves information from patch \(i\) and \(m^{(j)}_i = 0\) excludes it, and  (2) a parameter  \(r < 0\) controling the suppression strength: the larger \(\lvert r \rvert\), the stronger the masking effect. Forming the similarity matrix $\G=\Q{\K}^{\top}$, the attention focus on the information defined by $m^{(j)}$ is: \begin{equation}\label{eq:modified_similarity_aligned}
\begin{aligned}
    \tilde{\G} = & \G + \bigl(\mathbf 1 - M^{(j)}\bigr) r,
\end{aligned}
\end{equation}
 where $M^{(j)} = [m^{(j)},\dots,m^{(j)}]^\top.$ 
This operation drives attention scores in $\tilde{\G}$ for patches with $m_i\approx0$ down to $r$, thus masking them in the following softmax. The term $(\mathbf{1}-M^{(j)})$ adds 0 when the mask is 1 and $r$ when it is near 0. The masked attention weights $\softmax(\tilde{\mathcal{G}})$ yield the desired descriptor through the remaining layers:
\begin{equation}\label{eq:feature_descriptor}
F^{\prime(j)} \;=\;
\desc^{\ell\rightarrow}\!\bigl(\softmax(\tilde{\mathcal{G}})\V\bigr).
\end{equation}
Repeating this procedure for all $j$ reconstructs $F'\in\REAL^{H'\times W'\times D}$.

\paragraph{Mask design.}
For a patch \(j\), we set weight \(m^{(j)}_i=f\!\bigl(d_{ij}\bigr)\) where
\(d_{ij}\) is the distance in the patch grid \(\lVert(x_i,y_i)-(x_j,y_j)\rVert_2\), and $f$ as a hard cutoff \(\mathbf 1[d_{ij}\le\alpha]\), or soft decays \(2^{-d_{ij}}\) and \(1/d_{ij}\), giving flexible locality.

\subsection{Redundancy as low-rank structure}
\label{subsec:pcared}

\begin{figure*}[t]
    \centering
    \includegraphics[width=\linewidth]{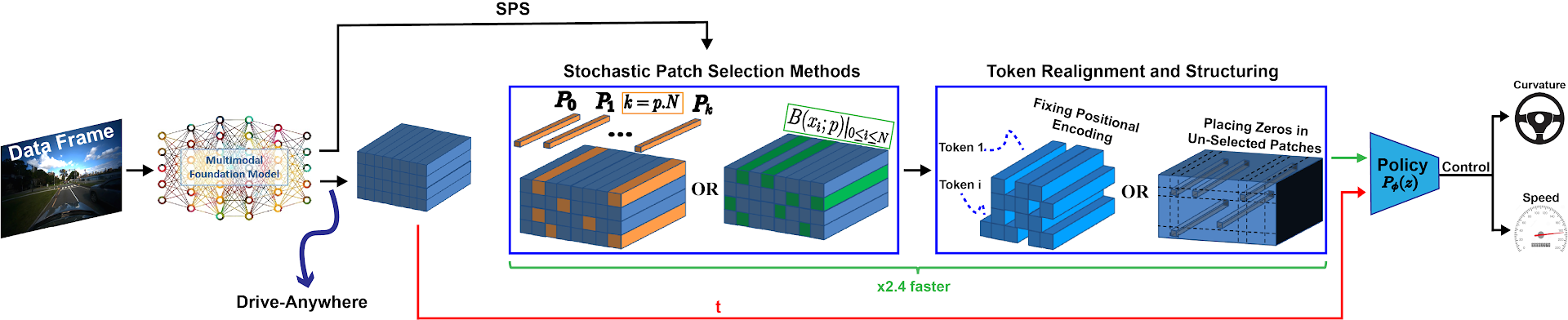}
\caption{\textbf{SPS algorithm vs Drive-Anywhere illustration}. From left to Right. 1st: Input images are processed through a frozen foundation model to produce patch-level descriptors. 2nd: In \textit{Drive-Anywhere}, the full tensor is forwarded unchanged to the policy. 3rd: In our approach, we introduce two patch selection strategies (uniform stochastic and matrix-based probability selection), followed by (4th) a restructuring phase: either masking unselected descriptors or removing them and adjusting positional encodings. Both versions preserve spatial semantics while significantly reducing runtime. \method{} improves efficiency by a factor of $\mathbf{2.4\times}$ while also boosting generalization.}
    \label{fig:SPS_vs_Drive_Anywhere}
\end{figure*}

Let $F_c \in \REAL^{N\times D}$ denote the matrix of patch descriptors obtained by applying $\desc$ to $F$, where $N=H'W'$ is the number of patches and $D$ is the dimension of each descriptor ($F_c$ is the reshaped $F'$ to matrix form).  
%the row-centered $F_c$ and 
Le $\tilde F_c$ be the mean-subtracted matrix, and $\lambda_1 \ge \cdots \ge \lambda_r>0$ be the eigenvalues of the sample covariance $\Sigma=\tfrac{1}{N}\tilde F_c^\top \tilde F_c$, with rank $r\leq d$.   The cumulative explained variance after the first $m\le r$ components is
\begin{align}
\label{em}
E(m):=\frac{\sum_{i=1}^{m}\lambda_i}{\sum_{i=1}^{r}\lambda_i}.
\end{align}

A rapidly saturating $E(m)$ for $m\ll d$ indicates a low-rank structure, hence redundancy among the patch descriptors. 
In Section~\ref{subsec:redundancy_empirics} we empirically show that BLIP2 features reach $90\%$ explained variance with $17$ of $64$ components ($E(17)>0.9$). Even more, when using half of the patches with the highest $\ell_2$ norm, the first $14$ principal components explain $90\%$ of the variance ($E(14)>0.9$).
A proof sketch and implementation details appear in Appx.~\ref{app:pca_corr_details}.

\subsection{Stochastic patch selection (SPS)}\label{sec:sps}
\paragraph{Reducing descriptor redundancy.} As we saw,
information about a scene is implicitly replicated across all tokens. The resulting redundancy
(i)~inflates the computation needed to extract \emph{every} patch feature, 
(ii)~encourages the policy to rely on spurious,
correlated cues that collapse under distribution shift, and
(iii)~ forces the policy to require every tiny detail of the full scene to provide a robust decision. 

Our goal is therefore to \emph{expose the policy}, at every input frame, to a fraction of the patch descriptors and withhold the others, via a token-level stochastic selection mechanism operating in the latent space of the foundation model, while maintaining the spatial layout of the remaining patches. Each random subset forms a different, yet still coherent, projection of the same scene, giving the policy a stream of stochastic but complete views. Consequently, the controller is trained to base its decisions on features that remain invariant to which specific tokens survive, yielding improved OOD robustness and lower compute, all while preserving spatial coherence and encouraging the use of robust, generalizable signals over exact token identity.

Let $N=H'W'$ denote the number of patch positions produced by the
backbone $\desc$ (see Sec.~\ref{subsec:prelim}).
Fix a \emph{sampling rate} $\rate\in(0,1]$.
For every input frame $t$ we:

\begin{enumerate}%[leftmargin=*,itemsep=2pt,wide]
\item \textbf{Sample a subset.}\;
      Uniformly at random, select exactly
      $k=\lceil\rate \cdot N\rceil$ patch indices,
      denoted $\Omega_t\subseteq\{1,\dots,N\}$ with $|\Omega_t|=k$.

\item \textbf{Compute descriptors.}\;
      Run the masked–attention extractor
      (Sec.~\ref{subsec:prelim}) \emph{only} for  patches 
      $i\in\Omega_t$,
      obtaining for every $i$, the descriptor $F'_{i}\in\REAL^{D}$.

\item \textbf{Build the sparse tensor.}\;
      Assemble $\tilde F'\in\REAL^{H'\times W'\times D}$:
      \begin{equation}\label{eq:sparse_tensor_build}
      \tilde F'_{i} =
      \begin{cases}
       F'_{i}, & i\in\Omega_t,\\%[1pt]
      \mathbf 0 , & \text{otherwise},
      \end{cases}
      \end{equation}
      thereby \emph{preserving the original spatial layout}
      so the policy head receives a tensor
      of fixed size.
\end{enumerate}

\paragraph{Computational footprint. } The expected fraction of evaluated descriptors is $\rate$, so feature extraction time scales linearly with the keep-rate.
For example, $\rate=0.4$ cuts the forward pass of the foundation model ViT\(_\text{L}\) backbone by \(\sim\!2.5\times\) without
altering the policy architecture. The variant which fix the positional embedding at the VIT, also improve the runtime of the policy.

\paragraph{Overall. } \method{} turns a dense, patch-aligned representation into a \emph{stochastically sparse}, yet spatially coherent, token sequence. By training the control policy on these variable but complete
projections of the scene, we (1) reduce compute, (2) discourage reliance on redundant correlations, and (3) drive consistent gains in OOD generalization; see Section~\ref{subsec:OOD&CMG}.

%--------------------------------------------------
\subsection{Ablations and variants}
\label{sec:ablations}
%--------------------------------------------------
To understand the design space, we evaluate two additional variants of the above procedure.

\paragraph{(v1) Threshold masking.}
Instead of sampling exactly $k$ indices,
draw an i.i.d.\ vector $R\sim\mathcal U(0,1)^{H'W'}$ and define
\begin{equation}\label{eq:threshold_masking}
\Omega_t=\bigl\{i\mid R_{i}\le\rate\bigr\}.
\end{equation}
Each patch is sampled independently with
probability~$\rate$, so the expected keep-rate is
$\rate$ but the actual count fluctuates.
This adds stochastic diversity at negligible cost.

\paragraph{(v2) Position-adjusted sparse sequence.}
After choosing $\Omega_t$ (either by fixed-count or threshold
masking), \emph{omit} the dropped tokens entirely and feed the policy
only the descriptors that were actually computed.  Each retained
descriptor $F'_{i}$ is augmented with its original positional
embedding $p_{i}\!\in\!\REAL^{D}$ drawn from the ViT’s full
positional-embedding table, so the controller still knows \emph{where}
each token came from.  This yields a variable-length input sequence
that is shorter by a factor of $\rate$ without inserting any zero
vectors.

Figure~\ref{fig:SPS_vs_Drive_Anywhere} illustrates our method alongside its variants.

\subsection{SPS preserves scene semantics}

\paragraph{Subspace preservation under uniform row sampling.} If the variance of $F_c$ concentrates in a few components, the patch descriptors (rows of $F_c$) lie near a low-dimensional subspace. If the row coherence is low, uniformly sampling enough rows preserves this principal subspace with high probability, hence the policy still receives descriptors that span the same semantics. We now show it formally.

\begin{lemma}[SPS preserves the row-space under low rank and bounded coherence]\label{lemma:1}
Let $F_c \in \mathbb{R}^{N\times d}$ be a centered data matrix with $\operatorname{rank}(F_c)=r$, and let its (thin) singular value decomposition be $F_c \;=\; U_r \Sigma_r V_r^\top,$ where $U_r \in \mathbb{R}^{N\times r}$ and $V_r \in \mathbb{R}^{d\times r}$ have orthonormal columns. The orthogonal projector onto the row space of $F_c$ is $\Pi_F \;=\; V_r V_r^\top,$  %\;=\; F_c^\top (F_c F_c^\top)^{\dagger} F_c. \]
and the row-space coherence is defined as $
\mu \;\triangleq\; \frac{N}{r}\, \max_{i\in [N]} \| e_i^\top U_r \|_2^2 \;\in\; [1,\, N/r].$
Let $\mathcal{I}\subset [N]$ be a uniformly random subset of $m$ indices without replacement, and let $F_{\mathcal{I}} \in \mathbb{R}^{m\times d}$ be the corresponding submatrix of $F_c$. Let $\Pi_{F_{\mathcal{I}}}$ denote the orthogonal projector onto the row space of $F_{\mathcal{I}}$. Then there exists a constant $C>0$ such that, for any $\varepsilon,\delta\in(0,1)$, if
\begin{equation}\label{eq:m}
% \[
m \;\ge\; C\,\frac{\mu\, r}{\varepsilon^2}\,\log\!\Bigl(\frac{r}{\delta}\Bigr),
% \]
\end{equation}

the following holds with probability at least $1-\delta$:
\begin{equation}\label{eq:pi}
% \[
\bigl\|\Pi_F - \Pi_{F_{\mathcal{I}}}\bigr\|_2 \;\le\; \varepsilon,
% \]
\end{equation}
i.e., the principal $r$-dimensional subspace of $F_c$ is preserved. For the full proof, see section~\ref{app:pca_corr_proof}.

\end{lemma}

%% file: sec/4_experimental_results.tex
\section{Experimental results}\label{sec:exp}

\begin{figure*}
    \centering
    \includegraphics[width=\linewidth]{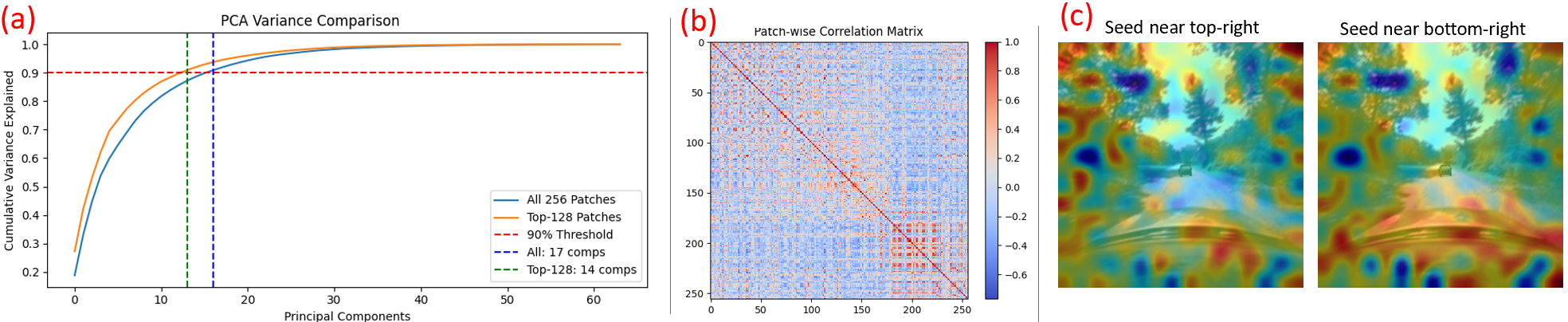}
    \caption{(a) Cumulative explained variance over principal components for all patches versus the top-$128$ patches selected by $\ell_2$ norm. The red line marks $90\%$. Vertical markers indicate $17$ and $14$ components for all and top-$128$, respectively. (b) Patch-wise Pearson correlation matrix for one scenario. Strong off-diagonal correlation indicates widespread cross-patch redundancy. (c)  Cosine-similarity overlays projected onto the image plane. Bright regions indicate patches whose descriptors are highly similar to the seed, visualizing global entanglement from self-attention.}
    \label{fig:panel1}
\end{figure*}
We begin by empirically confirming that patch-wise features are redundant and highly correlated.

\subsection{Redundancy analysis: qualitative and quantitative evidence}
\label{subsec:redundancy_empirics}

\paragraph{PCA explained variance over feature dimensions.}
Let $F \in \REAL^{N \times d}$ be the descriptor matrix. We visualize the cumulative explained variance $E(m)$ (as in~\eqref{em}) as a function of the number of components used $m$. Fig.~\ref{fig:panel1}(a) shows that $90\%$ of the variance is captured by $17$ of $64$ principal components. Repeating the analysis over the $128$ patches with the highest norm, yields $90\%$ with $14$ components, which confirms that redundancy persists even among the strongest tokens. The experiment was conducted on $10000$ frames from different scenes, and averaged across all.

\paragraph{Patch-wise correlation structure.}
We compute the Pearson correlation between all patch descriptors within a frame to obtain an $N \times N$ correlation matrix. A representative heatmap appears in Fig.~\ref{fig:panel1}(b), which shows extensive off-diagonal positive and negative correlation, indicating pervasive redundancy across spatial tokens. 

\paragraph{Similarity overlays on the image plane exampe.}
For a seed patch $i$ with descriptor $f_i \in \REAL^d$, we compute cosine similarities $s_j = \frac{f_i \cdot f_j}{\lVert f_i \rVert \lVert f_j \rVert}$ for all patches $j$, reshape $s \in \REAL^N$ to the $H' \times W'$ grid, upsample to the input resolution, and overlay it on the raw image with a heatmap. High intensity indicates semantic overlap. Examples for two seed locations are shown in Fig.~\ref{fig:panel1}(c). Widespread high similarity far from the seed in multiple locations illustrates global entanglement and redundancy.  See Appx.~\ref{app:sim_overlays} for the full $16 \times 16$ grid of overlays for all patches.

These qualitative and quantitative results support our hypothesis that patch descriptors live in a low-dimensional subspace and exhibit strong cross-patch redundancy, explaining why stochastic sub-sampling maintains scene semantics.

\begin{figure*}[!t]
    \centering
    \includegraphics[width=\linewidth]{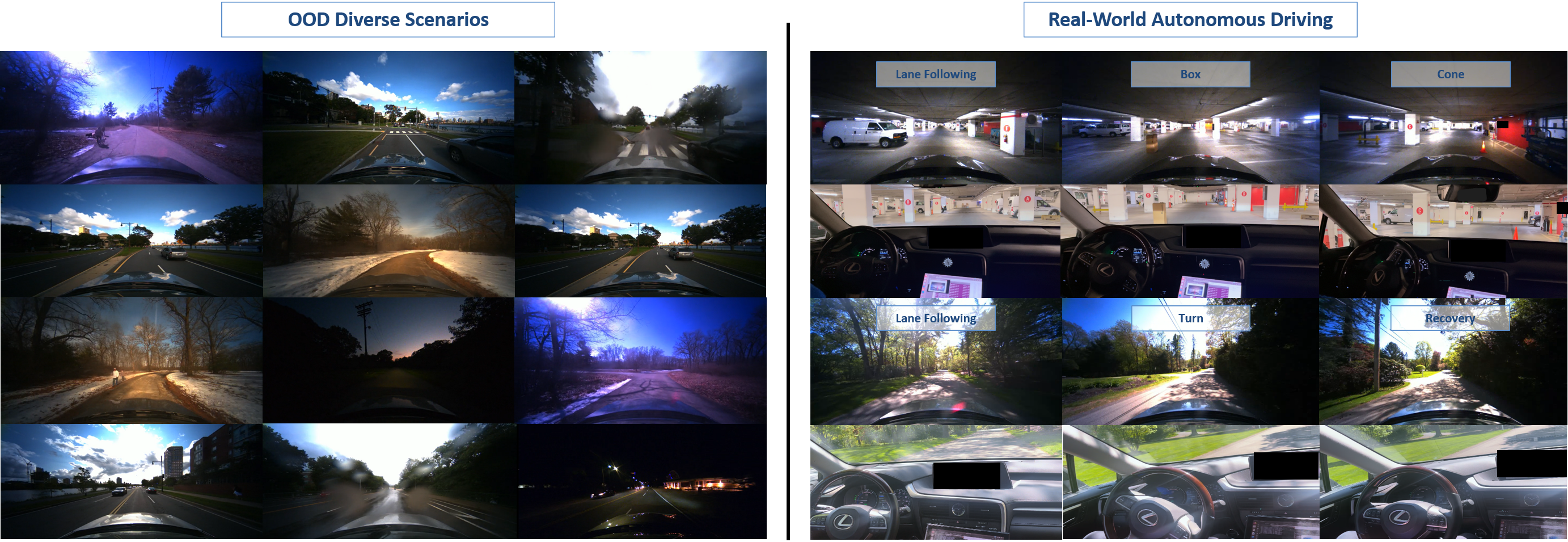}
\caption{Left: Diverse OOD sample frames across varying seasons, weather, and lighting conditions. Right: Real-car deployment representative frames from the rural road (public park) and parking garage, captured from both the onboard camera and external view.}
    \label{fig:AV_OOD}
\end{figure*}

\subsection{SPS Experimental settings}
We conduct extensive closed-loop driving experiments both for in-distribution and OOD environments. Below, we detail the hardware, tasks, evaluation metrics, and training data.

\paragraph{Task definition and evaluation protocol.}
As in~\cite{Wang2024Drive}, we target a general-purpose autonomous driving task requiring the vehicle to follow lanes and avoid obstacles. Failure is defined by three conditions: (i) crossing lane boundaries, (ii) collisions or dangerously close proximity to objects, and (iii) heading deviations exceeding 30$^\circ$ from the lane direction. In simulation, we assess performance using a normalized success duration metric, measuring how long the vehicle drives without triggering any failure, averaged over 100 episodes of roughly 20 seconds each. In real-world testing, performance is quantified by counting safety driver interventions, following the same failure definitions. Evaluations are performed in a closed-loop manner unless otherwise specified.

\paragraph{Training data and learning framework.}
Our training data combines real-world driving logs with diverse simulated experiences generated using VISTA \citep{amini2022vista}, a simulator built upon approximately two hours of real driving data captured across varied environments, lighting, and weather conditions. The learning procedure adopts Guided Policy Learning \citep{levine2013guided,amini2022vista}, leveraging privileged simulator signals to supervise image-based control policies. Control labels are derived using a PID controller for lane keeping and Control Barrier Functions (CBFs) \citep{Xiao2019} for safety-aware obstacle avoidance. 
Policies are trained using the Adam optimizer (learning rate $10^{-3}$), an L2 loss objective, and a plateau scheduler with patience 10 and no decay. 
We utilize BLIP-2 \citep{li2023blip} as the feature extractor to enable a fare comparison with SOTA~\citep{Wang2024Drive}.

\paragraph{Vehicle platform and Hardware.}
The real car experiments were conducted using a fully autonomous 2019 Lexus RX 450H equipped with high-performance computing and sensing hardware. The onboard system includes an NVIDIA RTX 4070 Ti GPU and an AMD Ryzen 7 3800X 8-core CPU. Visual input is captured via a BFS-PGE-23S3C-CS camera running at 30 frames per second, with a 130$^\circ$ horizontal field of view and a resolution of 960×600 pixels.
The models were trained and evaluated on a university-managed cloud cluster using a total of four A100 GPUs (40GB each). Training took approximately four days.

\subsection{OOD generalization}\label{subsec:OOD&CMG}
% - OOD Data Visualization -

\begin{table*}[t]
\centering
\caption{\textbf{Benchmarking OOD generalization.} $^\dagger$Indicates car types different from training. \textit{ID} = in-distribution. \textit{OOD} = out-of-distribution.}
\adjustbox{max width=\textwidth}{
\begin{tabular}[t]{@{}c|ccccc|c|c|ccc@{}}  % The high resolution columns
\toprule
\multirow{3}{*}{\textbf{Setting}} & \multicolumn{5}{c|}{\textbf{Scenarios}} 
& \multicolumn{5}{c}{\textbf{Methods}} \\

& & & & & 
& 
& \multicolumn{1}{c|}{\textbf{Uni-modal FMF}} 
& \multicolumn{3}{c}{\textbf{Multi-modal FMF}} \\

& Scene & Season & Weather & Time & Actor 
&  
\textbf{No-FM} 
& \multicolumn{1}{c|}{\textbf{I-ViT}} 
& \textbf{MF} & \textbf{DA} & \textbf{SPS (Ours)} \\

\midrule
ID & Rural & Summer & Dry & Day & Car & 1.00 & 1.00 & 0.72 & 1.00 & 1.00 \\
\midrule

\multirow{14}{*}{OOD} & \multirow{6}{*}{Rural} & Spring & Dry & Day & Car$^\dagger$ & 0.84 & 0.86 & 0.42 & 0.96 & 0.98 \\
&  & Summer & Dry & Night & Car$^\dagger$ & 0.30 & 0.80 & 0.35 & 0.89 & 0.92 \\
& & Fall & Dry & Day & Car$^\dagger$ & 0.90 & 0.95 & 0.74 & 0.91 & 0.96 \\
&  & Winter & Snow & Day & Car$^\dagger$ & 0.14 & 0.88 & 0.42 & 0.96 & 0.99 \\
& & Spring & Dry & Day & Animal & 0.85 & 0.89 & 0.39 & 0.95 & 0.99 \\
&  & Summer & Dry & Night & Animal & 0.29 & 0.59 & 0.39 & 0.85 & 0.86 \\
& & Fall & Dry & Day & Animal & 0.87 & 0.95 & 0.71 & 0.88 & 0.93 \\
&  & Winter & Snow & Day & Animal & 0.15 & 0.87 & 0.45 & 0.95 & 0.99 \\
\cmidrule{2-6}
& \multirow{6}{*}{Urban} & Summer & Dry & Day & Car$^\dagger$ & 0.55 & 0.77 & 0.50 & 0.62 & 0.82 \\
&  & Summer & Rain & Day & Car$^\dagger$ & 0.69 & 0.81 & 0.43 & 0.81 & 0.83 \\
&  & Summer & Dry & Night & Car$^\dagger$ & 0.45 & 0.81 & 0.42 & 0.78 & 0.87 \\
&  & Summer & Dry & Day & Animal & 0.58 & 0.80 & 0.50 & 0.64 & 0.79 \\
&  & Summer & Rain & Day & Animal & 0.66 & 0.83 & 0.43 & 0.78 & 0.84 \\
&  & Summer & Dry & Night & Animal & 0.45 & 0.86 & 0.36 & 0.81 & 0.88 \\

\midrule
\textbf{Average} &  &  &  &  &  & \textbf{0.55} & \textbf{0.83} & \textbf{0.47} & \textbf{0.84} & \textbf{0.90} \\
\midrule
\textbf{\makecell{Our increase (\%)\\vs. other methods}} &  &  &  &  &  & \textbf{\boldmath{$\uparrow 35\%$}} & \textbf{\boldmath{$\uparrow 7\%$}} & \textbf{\boldmath{$\uparrow 53\%$}} & \textbf{\boldmath{$\uparrow 6\%$}} & \textbf{\textendash} \\
\bottomrule
\end{tabular}
}
\label{tab:ood_generalization_SPS_50}
\end{table*}

To contextualize our improvements, we adapt the OOD generalization experiments and benchmarks reported in the \textit{Drive-Anywhere} \citep{Wang2024Drive} framework.
Specifically, we train the model in rural environments during summer, under dry weather conditions and daytime lighting, with the presence of other vehicles on the road. We then evaluate its performance across diverse scenes, weather conditions, times of day, and in the presence of other dynamic actors.

\paragraph{Baselines: }
(i) No Foundation Model (\textit{No-FM}) \citep{wang2023learning}, \citep{amini2022vista}, a baseline that trains a convolutional model without using foundation models (transformer-based variants performed similarly);  (ii) \textit{Mask-based Features (MF)} \citep{maalouf2023follow,zhao2023fast}, which segment the input image by applying a universal segmentation model~\citep{kirillov2023segment}, extracts a global feature vector for each region by applying a pretrained encoder~\citep{radford2021learning} on a crop bounding this region, and assign that vector uniformly to all pixels within the mask. (iii) \textit{Inherent ViT Features (I-ViT)} \citep{amir2021deep}, which extract per-patch features from the output of intermediate layers of a ViT model~\citep{zhang2022dino}, using the key, query, and value matrices as token-aligned visual descriptors, (iv) the current SOTA Drive Anywhere (DA)~\citep{Wang2024Drive}, which creates per-patch descriptor via the masking strategy explained in Section~\ref{subsec:prelim}, and ours; \textit{SPS} as defined in Section~\ref{sec:sps} without adding any variant from~\ref{sec:ablations}, with a \rate{} = 0.5.

\paragraph{Discussion.} The results are reported in Table~\ref{tab:ood_generalization_SPS_50}. First, \textit{MF} underperforms across both in-distribution and OOD settings, likely due to the applied masking model, which may miss and not segment part of the image~\citep{Dai2015-az}, and thus exclude relevant information from the feature tensor. In contrast, \textit{I-ViT} and \textit{DA} outperform \textit{No-FM}, underscoring the benefit of pretrained representations. However, \textit{I-ViT} remains a uni-modal approach and does not incorporate language grounding. Finally, our proposed method, \textit{SPS}, clearly outperforms all baselines, improving upon the previous SOTA (\textit{DA}) by an average of 6\% and up to 20\% in specific OOD scenarios. Our improvement over each method is reported in the last row of the table. We provide examples in Fig.~\ref{fig:AV_OOD} for representative frames from diverse OOD conditions.

\paragraph{Other variants. } We evaluate additional variants of \textit{SPS} yielding nine different models in total, incorporating different architectural modifications (e.g., different $\rate$, position-adjusted, threshold masking). Figure~\ref{fig:accuracy_vs_selection} shows that 8 out of the 9 models outperform the strongest baseline in OOD settings, and Table~\ref{tab:ood_generalization_all_models} shows detailed results of all variants across each OOD scenario, demonstrating the robustness and flexibility of our approach; see Section~\ref {subsec:run_efficiency_via_SPS} for details.

\paragraph{Cross-backbone generalization (DINO).}
We also applied SPS to a DINO backbone to test transfer beyond BLIP2. For brevity, we report in the appendix OOD scenarios that vary a lot from the in-distribution training set and are therefore especially challenging. On this subset, SPS improves DINO in $6$ of $7$ scenarios and ties in one, with an average absolute gain of $+3.3\%$; see Table~\ref {tab:dino_animal} in the Appendix.

\begin{table*}[t]
\centering
\caption{\textbf{Ablating all of our variants for OOD generalization.} $^\dagger$indicates car types different from training. \textit{ID} is in distribution. \textit{OOD} is out-of-distribution.}
\adjustbox{max width=\textwidth}{
    \begin{tabular}[t]{@{}c|ccccc|cccccccccc@{}}
    \toprule
    \multirow{2}{*}{\textbf{Setting}} & \multicolumn{5}{c}{\textbf{Scenarios}} & \multicolumn{5}{c}{\textbf{Methods}} \\
    & Scene & Season & Weather & Time & Actor & DA & MSPPS-70\% & MSPPS-50\% & MSPPS-30\% & SPPS-70\% & SPPS-50\% & SPPS-30\% & SPS-70\% & SPS-50\% & SPS-30\% \\
    \midrule
    ID & Rural & Summer & Dry & Day & Car & 1.00 & 1.00 & 1.00 & 1.00 & 1.00 & 1.00 & 1.00 & 1.00 & 1.00 & 1.00 \\
    \midrule
    \multirow{14}{*}{OOD} & \multirow{6}{*}{Rural} & Spring & Dry & Day & Car$^\dagger$ & 0.96 & 0.96 &  0.99 & 0.97 & 0.98 & 0.98 & 0.99 & 0.99 & 0.98 & 0.94 \\
    &  & Summer & Dry & Night & Car$^\dagger$ & 0.89 & 0.82 & 0.89 & 0.87 & 0.87 & 0.91 & 0.84 & 0.93 & 0.92 & 0.91 \\
    & & Fall & Dry & Day & Car$^\dagger$ & 0.91 & 0.98 & 0.94 & 0.97 & 0.95 & 0.97 & 0.97 & 0.96 & 0.96 & 0.98 \\
    &  & Winter & Snow & Day & Car$^\dagger$ & 0.96 & 0.95 & 0.98 & 0.99 & 0.99 & 0.96 & 0.99 & 0.96 & 0.99 & 0.98 \\
    & & Spring & Dry & Day & Animal & 0.95 & 0.95 & 0.97 & 0.97 & 0.97 & 0.98 & 0.95 & 0.95 & 0.99 & 0.95 \\
    &  & Summer & Dry & Night & Animal & 0.85 & 0.79 & 0.84 & 0.77 & 0.86 & 0.84 & 0.80 & 0.87 & 0.86 & 0.87 \\
    & & Fall & Dry & Day & Animal & 0.88 & 0.99 & 0.91 & 0.94 & 0.93 & 0.97 & 0.92 & 0.97 & 0.93 & 0.92 \\
    &  & Winter & Snow & Day & Animal & 0.95 & 0.94 & 0.97 & 0.99 & 0.99 & 0.97 & 0.95 & 0.97 & 0.99 & 0.97 \\
    \cmidrule{2-6}
    & \multirow{6}{*}{Urban} & Summer & Dry & Day & Car$^\dagger$ & 0.62 & 0.53 & 0.70 & 0.75 & 0.79 & 0.65 & 0.77 & 0.71 & 0.82 & 0.72 \\
    &  & Summer & Rain & Day & Car$^\dagger$ & 0.81 & 0.65 & 0.85 & 0.91 & 0.83 & 0.84 & 0.82 & 0.82 & 0.83 & 0.86 \\
    &  & Summer & Dry & Night & Car$^\dagger$ & 0.78 & 0.71 & 0.89 & 0.90 & 0.88 & 0.84 & 0.86 & 0.79 & 0.87 & 0.85 \\
    &  & Summer & Dry & Day & Animal & 0.64 & 0.50 & 0.65 & 0.72 & 0.76 & 0.64 & 0.74 & 0.73 & 0.79 & 0.70 \\
    &  & Summer & Rain & Day & Animal & 0.78 & 0.65 & 0.86 & 0.91 & 0.87 & 0.86 & 0.79 & 0.86 & 0.84 & 0.87 \\
    &  & Summer & Dry & Night & Animal & 0.81 & 0.71 & 0.88 & 0.90 & 0.89 & 0.85 & 0.86 & 0.85 & 0.88 & 0.90 \\
    \midrule
\textbf{Average} &  &  &  &  &  & $84.2\%$ & $79.5\%$ & $88\%$ & $89.7\%$ & $89.7\%$ & $87.6\%$ & $87.5\%$ & $88.3\%$ & $90.4\%$ & $88.7\%$ \\

\midrule
\textbf{\makecell{Our increase (\%)\\vs. other methods}} &  &  &  &  &  & \textbf{\textendash} & \textbf{\boldmath{$\downarrow 4.7\%$}} & \textbf{\boldmath{$\uparrow 3.8\%$}} & \textbf{\boldmath{$\uparrow 5.5\%$}} & \textbf{\boldmath{$\uparrow 5.5\%$}} & \textbf{\boldmath{$\uparrow 3.4\%$}} & \textbf{\boldmath{$\uparrow 3.3\%$}} & \textbf{\boldmath{$\uparrow 4.1\%$}} & \textbf{\boldmath{$\uparrow 6.2\%$}} & \textbf{\boldmath{$\uparrow 4.5\%$}} \\
\bottomrule
    \end{tabular}
}

\label{tab:ood_generalization_all_models}
\end{table*}

\subsection{Latent space text-augmented fine-tuning} \label{subsec:text_aug}

We adapt language-guided latent space augmentation to further improve the robustness of our best model, \textit{SPS-50\%}. The augmentation pipeline proceeds as follows:  (i) We prompt an LLM to produce a concise list of driving-relevant textual features; visual concepts present in the training scenes, e.g., \textit{Tree} or \textit{Truck}; these are then identified, and will plausibly be replaced to foster OOD generalization.
(ii) An LLM is queried to generate alternative concepts that are visually plausible and commonly encountered in driving scenarios (e.g., house instead of tree).   
(iii) In selected frames, patch-level descriptors associated with the target concepts are replaced in latent space with those derived from the suggested alternatives using the foundation model’s text encoder. This results in a form of data augmentation that preserves scene coherence while introducing semantic diversity.

\begin{table}[t]
\caption{
\textbf{Language-guided augmentation for OOD generalization.} We fine-tune the pretrained \textit{SPS-50\%} model using language-driven latent feature substitution based on LLM-suggested concepts.}
    \centering    
    \fontsize{9pt}{9pt}\selectfont
    \begin{tabular}{c|c|c|c|c}
        \toprule
        \textbf{RSDDC} & \textbf{RSDNC} & \textbf{RFDDC} & \textbf{RWSDC} & \textbf{RSDDA} \\
        \hline
        \textbf{\textendash} & +0.29\% & +1.85\% & \textbf{\textendash} & \textbf{\textendash} \\
        \hline
        \textbf{RSDNA} & \textbf{RFDDA} & \textbf{RWSDA} & \textbf{USDDC} & \textbf{USRDC} \\
        \hline
         -0.21\% & +2.93 & \textbf{\textendash} & -1.16 & +4.39\% \\
        \hline
        \textbf{USDNC} & \textbf{USDDA} & \textbf{USRDA} & \textbf{USDNA} & \textbf{All} \\
        \hline
        +0.36\% & +3.86\% & +3.28\% & +1.43\% & \footnotesize{\textbf{+1.7\%}} \\
        \bottomrule
    \end{tabular}
    \label{tab:data_aug}
\end{table}

We apply this technique as a lightweight fine-tuning step to the pretrained \textit{SPS-50\%} model. The model is updated for a small number of iterations using augmented latent features.
As shown in Table~\ref{tab:data_aug}, this results in a further +1.7\% accuracy gain across OOD scenarios, excluding cases where the base model already achieves near-perfect performance (\(\geq\)98\%), where meaningful further gains are unlikely. These results show that \textit{SPS} effectively integrates semantic augmentations through an interpretable, text-driven process, enabling scalable generalization via concept-level latent edits.

\subsection{Ablations: efficiency and performance} \label{subsec:run_efficiency_via_SPS}
Beyond improving generalization, \method{} offers substantial computational benefits during both training and inference. By selecting only a subset of patch features, we reduce the overall processing load without sacrificing semantic richness. We explore multiple variants leveraging the \method{} mechanism.

\paragraph{Variants.} In the base \textit{SPS} variant, unselected patch descriptors are replaced with zeros, preserving the feature tensor's original shape. In the more aggressive \textit{SPPS (Structured Patch Pruning with Selection)}, unselected tokens are entirely removed; the selected tokens are then spatially reorganized and assigned position embeddings relative to their original location (Method V2 from Section~\ref{sec:ablations}). 
A third variant, \textit{MSPPS (Matrix-based Structured Patch Pruning)}, stochastically prunes patch tokens based on a per-patch probability mask, yielding dynamic, matrix-based sparsity. Notably, it also prunes unselected patches as in SPPS (the combination of V1 and V2 from section~\ref{sec:ablations}).

\paragraph{Speed reported results.}  Each variant is evaluated across three selection rates: 70\%, 50\%, and 30\%.  Inference timing was measured over 100 independent runs per model component, with averages reported for the QFormer, vision projection and normalization, BLIP2 convolution, and the ViT policy model. While we train and benchmark nine model variants in total, only six representative models are shown in Table~\ref{tab:Runtime_Efficiency_via_SPS}, as the \textit{SPPS} and \textit{MSPPS} variants are functionally similar in terms of runtime, with any differences being negligible in practice. 
The \textit{SPS-50\%} model (our best model in terms of accuracy) achieves a $2.43\times$ speedup over the SOTA \textit{Drive-Anywhere}. 
At the more aggressive 30\% selection rate, \textit{SPPS} and \textit{\method{}} reach $3.31\times$ and $3.47\times$ speedups, respectively.   This shows the core trends: a consistent correlation between lower selection rates and improved runtime, highlighting that \method{} not only improves generalization, but also enables substantial computational gains, making it attractive for real-world autonomous systems deployment.

% Time-Accuracy Efficiency Graph

\paragraph{Speed vs accuracy reported results. } To analyse the trade-off between efficiency and accuracy, in Figure~\ref{fig:teaser}(4) we show a comparison, where each model is represented by a single node: The x-axis reflects runtime cost, defined as the inverse of the model’s speedup over the baseline - $1 / T_{\text{Factor}}$ (lower is better), while the y-axis reports average closed-loop performance.
Node color encodes a composite performance score, computed as accuracy divided by runtime (higher is better), highlighting models that strike the best balance between precision and efficiency.   
The six reported models gain better accuracy while being faster than the drive anywhere model; the improvements of our variant in terms of accuracy vary from $3.3\%$ to $6.2\%$, while the speed improved $1.69\times$ to $3.47\times$.
While our composite performance analysis (Fig.~\ref{fig:teaser}(4)) shows that \textit{SPS-70\%} and \textit{SPPS-70\%} achieve the highest scores when accuracy and runtime are equally weighted, practical deployment considerations suggest a different weighting scheme. 
In the context of real-world autonomous driving, reliability and generalization, particularly under OOD conditions, are of higher criticality than marginal gains in computational speed. The \textit{SPS-50\%} model consistently delivers the highest average accuracy across scenarios, and this performance margin is non-trivial in safety-critical systems. We therefore view \textit{SPS-50\%} as the strongest overall candidate for real-world deployment, offering the best tradeoff between robustness and efficiency.

\paragraph{The effect of \textbf{\rate}. }
As shown in Table~\ref{tab:ood_generalization_all_models}, most variants outperform the \textit{Drive-Anywhere} benchmark across OOD scenarios, with \textit{SPS-50\%} achieving the highest accuracy (+6.2\%), highlighting the general effectiveness of stochastic patch selection.
%To visualize the interaction between selection rate and performance, 
Figure~\ref{fig:accuracy_vs_selection} plots performance against selection rate for each variant, revealing distinct trends: \textit{SPS} peaks at 50\%, offering the strongest gain; \textit{SPPS} performs best at 70\% and declines with more aggressive pruning; and \textit{MSPPS} improves as selection becomes stricter, peaking at 30\%. These findings confirm that stochastic selection improves generalization across architectures, though the optimal ratio varies. Notably, moderate selection rates strike the best tradeoff between redundancy reduction and information retention in models where spatial structure is retained.

\begin{table*}[t]
  \centering
  \caption{\textbf{Runtime Efficiency of \method{}.} The calculations were averaged across $100$ independent runs.}
  \label{tab:Runtime_Efficiency_via_SPS}
  
  \resizebox{\textwidth}{!}{%
  \begin{tabular}{@{}c|cc|cc|cc|cc|c|c@{}}
  \toprule
  \multirow{2}{*}{\textbf{Model}} 
  & \multicolumn{2}{c|}{\textbf{Qformer}} 
  & \multicolumn{2}{c|}{\textbf{Vision Proj.\&Norm}} 
  & \multicolumn{2}{c|}{\textbf{FM Project}} 
  & \multicolumn{2}{c|}{\textbf{Policy}} 
  & \multicolumn{2}{c}{\textbf{Total Efficiency}} \\
  & Qformer & Efficiency & VP\&N & Efficiency & FMC & Efficiency & Policy & Efficiency & Percentage & Factor \\
  \midrule
  Benchmark & 2.17s & \textendash & 871$\mu s$ & \textendash & 540$\mu s$ & \textendash & 17$ms$ & \textendash & 100$\%$ & x1 \\
  SPS\_$70\%$ & 1.56s & $39.10\%$ & 839$\mu s$ & $3.75\%$ & 376$\mu s$ & $43.67\%$ & 20$ms$ & $0.24\%$ & $86.52\%$ & x1.87 \\
  SPS\_$50\%$ & 1.11s & $95.50\%$ & 855$\mu s$ & $1.88\%$ & 370$\mu s$ & $46.08\%$ & 20$ms$ & $-1.22\%$ & $143.45\%$ & x2.43 \\
  SPS\_$30\%$ & 0.71s & $205.63\%$ & 842$\mu s$ & $3.45\%$ & 390$\mu s$ & $38.41\%$ & 20$ms$ & $0.41\%$ & $247.50\%$ & x3.47 \\
  SPPS\_$70\%$ & 1.56s & $39.10\%$ & 846$\mu s$ & $2.93\%$ & 500$\mu s$ & $8.06\%$ & 14.3$ms$ & $18.92\%$ & $69.02\%$ & x1.69 \\
  SPPS\_$50\%$ & 1.11s & $95.50\%$ & 837$\mu s$ & $4.06\%$ & 493$\mu s$ & $9.41\%$ & 14.5$ms$ & $17.45\%$ & $126.41\%$ & x2.26 \\
  SPPS\_$30\%$ & 0.71s & $205.63\%$ & 841$\mu s$ & $3.58\%$ & 542$\mu s$ & $-0.44\%$ & 13.9$ms$ & $22.43\%$ & $231.21\%$ & x3.31 \\
  \bottomrule
  \end{tabular}
  }
\end{table*}

\begin{figure}
    \centering
    \includegraphics[width=\linewidth]{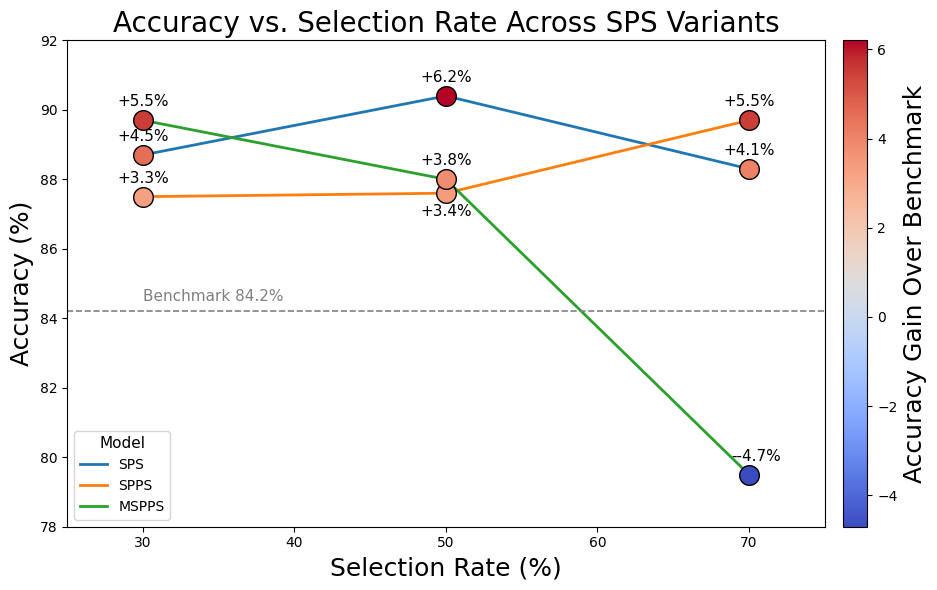}
    \caption{Accuracy as a function of selection rate across variants.}
    \label{fig:accuracy_vs_selection}
\end{figure}

\subsection{Real-world car deployment} \label{subsec:real_world_car_deployment}
We deployed the \textit{SPS-50\%} model on a full-scale autonomous vehicle (see supplementary video). The experiments took place during the spring season at two distinct locations: a rural road within a public park and an underground parking garage. 
In the rural scenario, the vehicle successfully performed lane-following under natural daylight conditions, confirming the model’s ability to transfer from simulation to the real world. The underground scenario was a challenging OOD test case, featuring low-light conditions, the absence of clear lanes, and visually cluttered geometry, such as perpendicular parking lines, parked vehicles, and distractor objects. Despite never encountering such a configuration at training, the vehicle executed the correct motion.
In addition, it successfully avoided static obstacles, including cones and boxes, further demonstrating its robustness. Representative frames from both real-world scenarios, captured from onboard and external views, are shown in Figure~\ref{fig:AV_OOD}. 

%% file: sec/5_conclusions_and_future_work.tex
\section{Conclusion and future work}
\label{sec:conclution_and_future_work}

This work shows that patch-aligned features (extracted from vision–language models) used for training driving policies are correlated and redundant. Motivated by this, it introduced \method{}: a stochastic patch–masking strategy that tackles the redundancy inevitably built into such features. By randomly suppressing a fraction of patch descriptors while preserving their spatial arrangement, \method{} forces the policy to ground its decisions on features that are robust to which tokens survive.  Across a diverse suite of closed-loop driving scenarios, covering weather, lighting and geographic domain shifts, \method{} delivers a \textbf{6.2\%} average gain in OOD success rate, peaks at \textbf{20.4\%} improvement in the hardest scenarios, and speeds inference by \textbf{2.4$\times$} compared with the previous SOTA.  
The same mechanism enables text-conditioned data augmentation in the latent patch space, yielding more \textbf{1.7\%} boost without extra image synthesis. Ablations confirm that the benefit is robust to (reasonable) selection rates, selection methods, and feature re-ordering schemes: 8 out of the 9 trained systems surpassed the prior SOTA.  We demonstrated that policies trained with \method{} in simulation transfer to a real-world autonomous vehicle with \emph{no} extra fine-tuning, highlighting the practical relevance of the approach.

\noindent\textbf{Future work} includes sophisticated selection techniques (i) learning a state-dependent sampling policy that select the number of patches to allocate based on scenes, and (ii) going beyond uniform sampling by \emph{inspecting} the descriptors themselves, e.g., coreset selection, attention entropy, or mutual-information scores, to drop provably redundant patches, further sharpening both efficiency and robustness.

%% file: sec/6_acknowledgments.tex
\section*{Acknowledgments}
Alaa Maalouf acknowledges support from the Neubauer Family Foundation and from the MAOF Fellowship of the Council for Higher Education. This work is supported by Toyota Research Institute (TRI) and Capgemini Engineering. It, however, reflects solely the opinions and conclusions of its authors and not TRI or any other Toyota entity.

%% file: sec/7_appendix.tex
\clearpage
\appendix

\section{Appendix}

\subsection{PCA and correlation analysis: definitions and rationale}
\label{app:pca_corr_details}

\noindent\textbf{Feature covariance.}
Let $F_c$ be an $N \times D$ matrix of the $N$ computed patch descriptors. Assume $F_c$ is centered (otherwise do so). Then compute $\Sigma = \frac{1}{N-1} F_c^\top F_c \in \REAL^{d \times d}$. The total variance equals the trace of $\Sigma$, which is the sum of eigenvalues. The cumulative explained variance curve reports $\sum_{i=1}^m \lambda_i \big/ \sum_{i=1}^d \lambda_i$.

\noindent\textbf{Top-energy subset.}
Let $f_j$ denote the computed descriptor for patch $j$. Compute $\lVert f_i \rVert_2$ for each patch, sort, and keep the top $128$ patches. Recompute the PCA curve on this subset. Achieving $90\%$ variance with about $14$ components indicates redundancy even among the strongest tokens.

\noindent\textbf{Why feature-space PCA supports patch masking.}
Although SPS masks patches rather than feature channels, a low-rank feature manifold implies that many patches project into the same few directions. Removing some rows of $F$ preserves the principal subspace with high probability, so semantics remain available from the retained tokens. One can also analyze patch-wise covariance $\frac{1}{d-1} F_c F_c^\top \in \REAL^{N \times N}$, which shows many correlated patches, consistent with the qualitative overlays in Appx.~\ref{app:sim_overlays}.

\medskip
\noindent\textbf{Second-moment viewpoint.}
The matrix $F_c^\top F_c$ is the second-moment (Gram) matrix of descriptors.
When we uniformly sample $m$ rows without replacement to form $SF_c$, the expected Gram matrix of the sampled descriptors satisfies
$\mathbb{E}\big[(SF_c)^\top(SF_c)\big] = \tfrac{m}{N}\,F_c^\top F_c$.
Thus a quadratic form induced by the sampled rows preserves the full quadratic form in expectation up to the known scale $m/N$.

\subsection{Full grid similarity overlays}
\label{app:sim_overlays}

\begin{figure*}[t]
    \centering
    \includegraphics[width=0.98\textwidth]{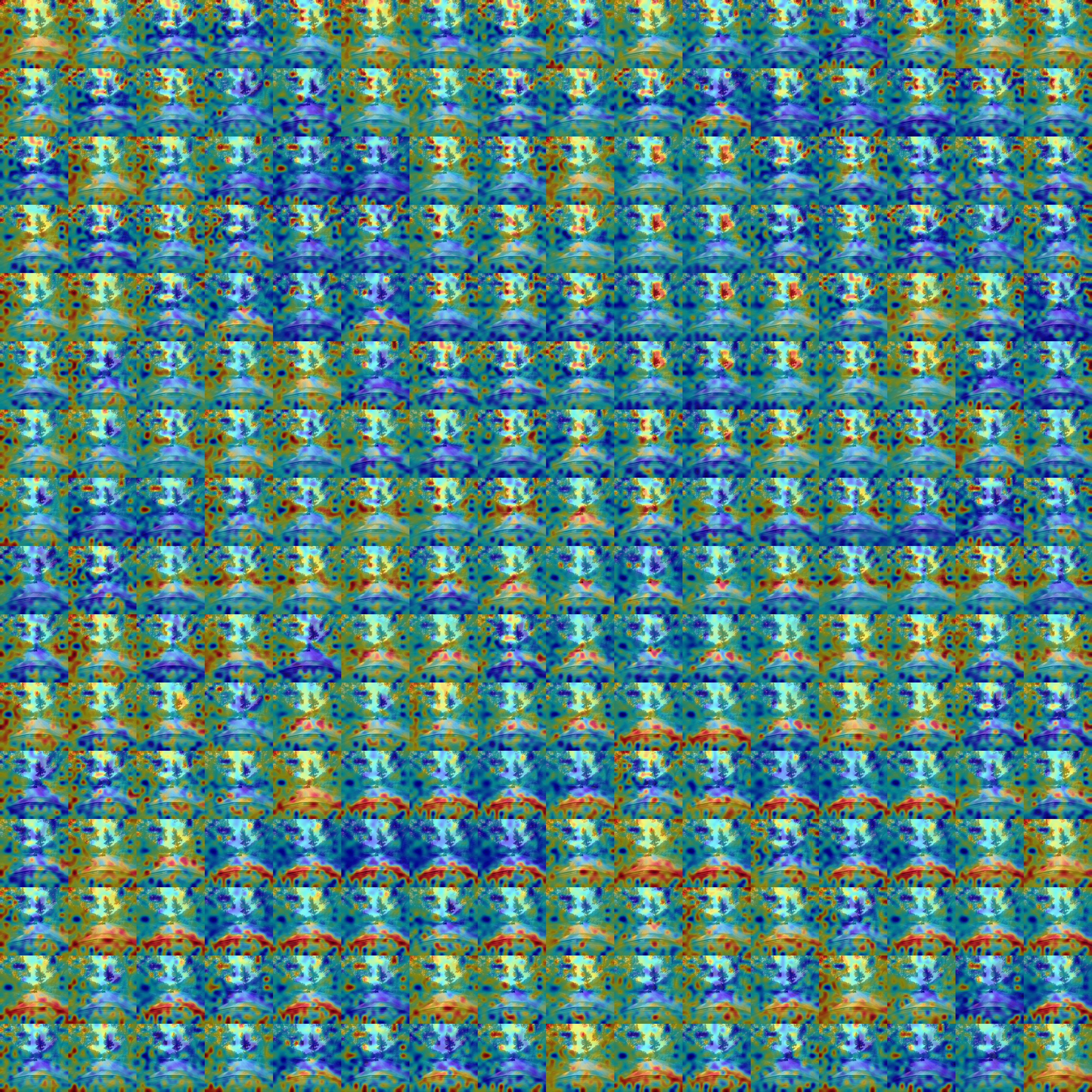}
    \caption{Full grid of similarity overlays for all $16 \times 16$ patches. Each cell shows the raw image overlaid with the cosine similarity from the corresponding seed patch to all other patches.}
    \label{fig:grid_overlay}
\end{figure*}

\noindent\textbf{How to compute it?} Given $F_c \in \REAL^{N \times d}$ be the matrix of patch features as before, and let $f_j$ denote the computed descriptor for patch $j$. Choose a seed index/patch $i$ and define $s_j = \frac{f_i \cdot f_j}{\lVert f_i \rVert \lVert f_j \rVert}$ for every patch $j$.  Set $s= (s_1, \cdots, s_N)$, and reshape $s$ to form $H' \times W'$ matrix, then upsample the matrix to the input resolution, and normalize to $[0,1]$ and overlay with the original input frame $F$.

\noindent \textit{Interpretation of Fig.~\ref{fig:grid_overlay}.} Blue denotes low or negative cosine similarity and red denotes high positive similarity. For a given seed patch, only a sparse set of regions are strongly similar (red), while most are weakly or negatively related (blue). As the seed changes across the grid, the red regions shift to locations that share semantics with the new seed, illustrating globally entangled yet redundant representations.

\subsection{DINO backbone results on Animal OOD scenarios}
\label{app:dino_animal}
We also evaluate SPS with the DINO backbone (to show generalization across different setups/backbones) and observe consistent gains (Table~\ref{tab:dino_animal}).

\begin{table}[t]
\centering
\caption{\textbf{SPS with a DINO backbone on Animal OOD scenarios.} These actors were not present in the in-distribution training set. SPS improves DINO in $6$ of $7$ scenarios and ties in one.}
\begin{tabular}{@{}lccc@{}}
\toprule
\textbf{Scenario} & \textbf{DINO} & \textbf{DINO + SPS} & \textbf{Gain} \\
\midrule
RSpDDA & 0.89 & 0.96 & $+7\%$ \\
RSuDNA & 0.59 & 0.65 & $+6\%$ \\
RFDDA  & 0.95 & 0.97 & $+2\%$ \\
RWSDA  & 0.87 & 0.92 & $+5\%$ \\
USuDDA & 0.80 & 0.82 & $+2\%$ \\
USuRDA & 0.83 & 0.83 & $0\%$ \\
USuDNA & 0.86 & 0.87 & $+1\%$ \\
\midrule
\textbf{Average} & $82.7\%$ & $86\%$ & \textbf{+3.3\%} \\
\bottomrule
\end{tabular}
\label{tab:dino_animal}
\end{table}

%%%%%%%%%%%%%%%%%%%%%%%%%%%%%%%%%%%%%%%%%%%%%%%%%%%%%%%%%%%%%%%%%%%%%%%%%%%%%%%%%%%%%%%%%%%%%%%%%%%%%%%%%

% --- New Proof ---

\subsection{Proof of Lemma~\ref{lemma:1}}
\label{app:pca_corr_proof}

\begin{proof}[Proof of Lemma~\ref{lemma:1}]
\textbf{Step 0: Thin SVD, projectors, and coherence.}
Let the thin SVD of the centered descriptor matrix be
$
F_c = U_r \Sigma_r V_r^\top
$,
with $U_r\in\mathbb{R}^{N\times r}$ and $V_r\in\mathbb{R}^{d\times r}$ having orthonormal columns, and $\Sigma_r\in\mathbb{R}^{r\times r}$ diagonal with positive entries.
The projector onto the row space is $\Pi_F = V_r V_r^\top$.
Define the uniform row coherence
$
\mu \triangleq \frac{N}{r}\,\max_{i\in[N]} \|(U_r)_{i:}\|_2^2 \in [1,\,N/r].
$
Fix tolerances $\varepsilon\in(0,1)$ and $\delta\in(0,1)$.

\medskip
\noindent
\textbf{Step 1: Second-moment identity in expectation.}
Let $S\in\{0,1\}^{m\times N}$ select $m$ rows uniformly without replacement, and write $SF_c$ for the sampled submatrix.
Then $S^\top S=\mathrm{diag}(c_1,\dots,c_N)$ with $c_i\in\{0,1\}$ indicating whether row $i$ is selected, so
$\mathbb{E}[S^\top S]=\tfrac{m}{N}I$ and
\begin{equation}
% \[
\mathbb{E}\big[(SF_c)^\top(SF_c)\big] = F_c^\top \mathbb{E}[S^\top S] F_c = \tfrac{m}{N}\,F_c^\top F_c.
% \]
\end{equation}

Hence the second-moment (Gram) matrix of sampled descriptors preserves the full one in expectation up to the factor $m/N$.

\medskip
\noindent
\textbf{Step 2: Subspace embedding via uniform row sampling.}
By matrix Chernoff concentration with bounded coherence (standard subspace-embedding results), if
\[
m \;\ge\; C\,\frac{\mu\,r}{\varepsilon^2}\,\log\!\Bigl(\frac{r}{\delta}\Bigr),
\]

then with probability at least $1-\delta$ we have the spectral sandwich
\begin{equation}\label{eq:sandwich}
(1-\varepsilon)\,I_r \;\preceq\; \frac{N}{m}\,U_r^\top S^\top S\,U_r \;\preceq\; (1+\varepsilon)\,I_r.
\end{equation}
In particular, $U_r^\top S^\top S\,U_r$ is positive definite, so $S U_r$ has full column rank $r$.
By \eqref{eq:sandwich}, all eigenvalues of $\tfrac{N}{m}\,U_r^\top S^\top S\,U_r$ lie in $[\,1-\varepsilon,\,1+\varepsilon\,]$, 
hence $U_r^\top S^\top S\,U_r \succ 0$. 
Positive definiteness implies that for any $x\neq 0$,
$x^\top U_r^\top S^\top S\,U_r x > 0$; therefore $SU_r x \neq 0$ and $\mathrm{rank}(SU_r)=r$. 
This means the sampled rows still span an $r$-dimensional subspace in $\mathbb{R}^N$.

\medskip
\noindent
\textbf{Step 3: Consequences for the row space.}
Using $F_c = U_r \Sigma_r V_r^\top$,
\begin{equation}
\begin{aligned}
(SF_c)^\top(SF_c) \;&=\; V_r \,\Sigma_r \,\bigl(U_r^\top S^\top S\,U_r\bigr)\,\Sigma_r \,V_r^\top \\
                  \;&=\; V_r\bigl[\Sigma_r (U_r^\top S^\top S\,U_r)\Sigma_r\bigr]V_r^\top.
\end{aligned}
\end{equation}

Since $U_r^\top S^\top S\,U_r$ is $r\times r$ positive definite, 
the nonzero eigenvectors of $(SF_c)^\top(SF_c)$ are the columns of $V_r$, and the corresponding eigenvalues are those of $\Sigma_r (U_r^\top S^\top S\,U_r)\Sigma_r$.
Thus $(SF_c)^\top(SF_c)$ has the same nonzero eigenvectors as $V_r$, whose columns span $\mathrm{span}(V_r)$.
Consequently, the range of $(SF_c)^\top(SF_c)$ equals $\mathrm{span}(V_r)$, so the row space of $SF_c$ coincides with that of $F_c$, and the orthogonal projectors satisfy $\Pi_{SF_c}=\Pi_F$.
This is stronger than the bound $\|\Pi_F - \Pi_{SF_c}\|_2 \le \varepsilon$, which follows immediately.

\medskip
\noindent

Combining Steps 1–3 proves the lemma under the stated sample complexity.
\qedhere
\end{proof}

\paragraph*{Remark 1 (Lipschitz stability).}
Let $\phi$ be a row-wise Lipschitz map with constant $L$ (for example, a per-token linear map followed by a 1-Lipschitz normalization). Let $\phi$ act independently on each row. Then
\begin{equation}
\begin{aligned}
\|\phi(SF_c) - \phi(F_c)\|_F^2 &= \sum_{i=1}^N \|\phi((SF_c)_{i:}) - \phi((F_c)_{i:})\|_2^2 \\
                               &\le \sum_{i=1}^N L^2 \|(SF_c)_{i:} - (F_c)_{i:}\|_2^2 \\
                               &= L^2 \|SF_c - F_c\|_F^2,
\end{aligned}
\end{equation}

so $\|\phi(SF_c) - \phi(F_c)\|_F \le L\,\|SF_c - F_c\|_F$.
Standard bounded-difference or matrix Bernstein arguments for sampling without replacement yield concentration of this deviation as $m$ grows.

\medskip
\noindent

\paragraph*{Remark 2 (Unbiased quadratic forms and covariance).}
For any per-token linear map $\phi(X)=XW$,
\begin{equation}
\mathbb{E}\big[((SF_c)W)^\top ((SF_c)W)\big]
= \tfrac{m}{N}\,(F_cW)^\top(F_cW).
\end{equation}

Thus quadratic objectives and covariances computed from the sampled rows preserve the full counterparts in expectation up to a known scale $m/N$.
One can reweight by $N/m$, or rely on normalization layers to absorb the scale.

\emph{Covariance remark.}
Let $\widehat{\Sigma}_m \triangleq \tfrac{1}{m-1}(SF_c)^\top(SF_c)$ be the usual unbiased sample covariance built from the sampled rows, and $\widehat{\Sigma}_N \triangleq \tfrac{1}{N-1}F_c^\top F_c$ the full-sample covariance.
Then
\begin{equation}
\begin{aligned}
\mathbb{E}\big[\widehat{\Sigma}_m\big] &= \frac{1}{m-1}\,\mathbb{E}\big[(SF_c)^\top(SF_c)\big] = \frac{m}{N(m-1)}\,F_c^\top F_c \\
                                       &= \frac{N-1}{N}\cdot\frac{m}{m-1}\;\widehat{\Sigma}_N.
\end{aligned}
\end{equation}

The factor $\alpha \triangleq \tfrac{N-1}{N}\cdot\tfrac{m}{m-1}$ is close to $1$ whenever $m$ and $N$ are moderate to large.
For $m=N/2$, $\alpha = \tfrac{N-1}{N}\cdot \tfrac{N}{N-2} = \tfrac{N-1}{\,N-2\,} = 1 + O(1/N)$.
Hence the covariance is preserved in expectation up to a negligible finite-sample correction.
In practice, either explicit reweighting ($N/m$) or standard normalization layers absorb any global scale, and the subspace preservation of Steps~2–3 is unaffected by scale.

% --- End Proof ---

\subsection{Self Attensin Causes the Patch Redundancy}
\noindent\textbf{Self-attention mixing.}Given a patch embedding matrix $X \in \REAL^{N \times D}$, a transformer layer forms $Q = X W^Q$, $K = X W^K$, and $V = X W^V$ with learned projections.
The attention logits are $G = Q K^\top$, and the token update is
\begin{equation}\label{eq:self_attention_update}
Y \;=\; \softmax(G)\,V \quad\in\; \REAL^{N \times D}.
\end{equation}
Hence, each output token $Y_i$ is a convex combination of all value vectors $V_j$ that are weighted by content similarity through the softmax of $Q_i K_j^\top$. This induces global mixing of information across tokens, so the descriptor at each spatial location carries scene-wide context.